%% file: clear2023.tex
\documentclass[final,12pt]{clear2023}
\input{preamble}

\title[Estimating long-term causal effects in the presence of unobserved confounding]{Estimating long-term causal effects from short-term experiments and long-term observational data with unobserved confounding}

\usepackage{times}
\usepackage{amsfonts}  
\usepackage{cleveref}
\usepackage{etoolbox}

\patchcmd{\theorem}{Theorem}{Proposition}{}{}

\clearauthor{%
 \Name{Graham Van Goffrier} \Email{ucapgwg@ucl.ac.uk}\\
 \addr Department of Physics and Astronomy, University College London\footnote{Research was started while this author as an intern at Spotify.}
 \AND
 \Name{Lucas Maystre} \Email{lucasm@spotify.com}\\
 \addr Spotify%
 \AND
 \Name{Ciarán Gilligan-Lee} \Email{ciaranl@spotify.com, ciaran.lee@ucl.ac.uk}\\
 \addr Spotify \& Department of Physics and Astronomy, University College London%
}

\begin{document}

\maketitle

\begin{abstract}%
Understanding and quantifying cause and effect is an important problem in many domains. The generally-agreed solution to this problem is to perform a randomised controlled trial. However, even when randomised controlled trials can be performed, they usually have relatively short duration's due to cost considerations. This makes learning long-term causal effects a very challenging task in practice, since the long-term outcome is only observed after a long delay. In this paper, we study the identification and estimation of long-term treatment effects when both experimental and observational data are available. Previous work provided an estimation strategy to determine long-term causal effects from such data regimes. However, this strategy only works if one assumes there are no unobserved confounders in the observational data. In this paper, we specifically address the challenging case where unmeasured confounders are present in the observational data. Our long-term causal effect estimator is obtained by combining regression residuals with short-term experimental outcomes in a specific manner to create an instrumental variable, which is then used to quantify the long-term causal effect through instrumental variable regression. We prove this estimator is unbiased, and analytically study its variance. In the context of the front-door causal structure, this provides a new causal estimator, which may be of independent interest. Finally, we empirically test our approach on synthetic-data, as well as real-data from the International Stroke Trial.

\end{abstract}


\section{Introduction}
\label{sec:introduction}

\input{01-introduction}

\section{Related work}
\label{sec:relatedwork}

\input{02-relatedwork}

\section{Methods}
\label{sec:theory}

\input{03-new}

\section{Experiments}
\label{sec:experiments}

\input{04-experiments}
\section{Conclusion}
\label{sec:conclusion}

\input{05-conclusion}

\acks{GVG was supported by Spotify and by the STFC UCL Centre for Doctoral Training in Data Intensive
Science (grant no. ST/P006736/1), and was funded by the UCL Graduate Research and Overseas Research Scholarships. This research began while GVG was an intern at Spotify. The authors thank Mounia Lalmas for supporting this project, and the anonymous reviewers for their valuable feedback.
}

\bibliography{clear2023.bib}

\appendix

\section{Covariance Algebra}
\label{sec:algebra}

\input{A1-algebra}

\section{Proofs of Estimator Biases and Variances}
\label{sec:proofs}

\input{A2-proofs}

\section{Nonlinear Bias Examples}
\label{sec:nonlinear}

\input{A3-nonlinear}

\end{document}

%% file: preamble.tex
\usepackage[english]{babel} 
\usepackage[T1]{fontenc} 
\usepackage[utf8]{inputenc} 
\usepackage[babel]{microtype} 

\usepackage{url}
\usepackage{csquotes}
\usepackage{float}
\usepackage{siunitx}

\usepackage{natbib}
\usepackage{algorithm}
\usepackage{algorithmicx}
\usepackage[noend]{algpseudocode}

\usepackage[belowskip=5pt,aboveskip=0pt]{caption}

\usepackage{tikz}
\usetikzlibrary{bayesnet,decorations.markings,automata,positioning}

\newcommand{\email}[1]{\href{mailto:#1}{\nolinkurl{#1}}}

\usepackage{listings}
\usepackage{color}

\newcommand{\EE}{\mathbb{E}}

\newcommand{\noise}[1]{\ensuremath{\sigma_{u_#1}^2}}
\newcommand{\ind}{\perp\!\!\!\!\perp} 

\renewcommand{\thefootnote}{\fnsymbol{footnote}}

\usepackage{hyperref}

%% file: 01-introduction.tex
Quantifying cause and effect relationships is of fundamental importance in many fields, from medicine to economics (\cite{richens2020improving, gilligan2020causing}). The gold standard solution to this problem is to conduct randomised controlled trials, or A/B tests. However, in many situations, such trials cannot be performed; they could be unethical, too expensive, or just  technologically infeasible. However, even when randomised controlled trials can be performed, they usually have relatively short durations due to cost considerations. For example, online A/B tests in industry usually last for only a few weeks \citep{gupta2019top}. This makes learning long-term causal effects a very challenging task in practice, since long-term outcomes are often observed only after a long delay. Often short-term outcomes are different to long-term ones \citep{kohavi2012trustworthy}, and, as many decision-makers are interested in long-term outcomes, this is a crucial problem to address. For instance, technology companies are interested in understanding the impact of deploying a feature on long-term retention \citep{chandar2022using}, economists are interested in long-term outcomes of job training programs \citep{athey2019surrogate}, and doctors are interested in the long-term impacts of medical interventions, such as treatments for stroke \citep{Carolei1997}.

In contrast to experimental data, observational data are often easier and cheaper to acquire, so they are more likely to include long-term outcome observations. Previous work by \citet{athey2019surrogate} devised a method to estimate long-term causal effects by combining observational long-term data and short-term experimental data. However, this strategy only works if one assumes there are no unobserved confounders in the observational data. Nevertheless, observational data are very susceptible to unmeasured confounding, which can lead to severely biased treatment effect estimates. Can we combine these short-term experiments with observational data to estimate long-term causal effects when latent confounders are present in observational data? 
    
In this paper, we address this problem and study the identification and estimation of long-term treatment effects when both short-term experimental data and observational data with latent confounders are available. We initially work with linear structural equation models. Our long-term causal effect estimator is obtained by combining regression residuals with short-term experimental data in a specific manner to create an instrumental variable, which is then used to quantify the long-term causal effect through instrumental variable regression. We prove that this estimator is unbiased, and analytically study its variance. When applied in the front-door causal structure, this strategy provides a new causal estimator, which may be of independent interest. We extend this estimator from linear structural causal models to the partially linear structural models routinely studied in economics \citep{chernozhukov2016double} and prove unbiasedness still holds under mild assumptions. Finally, we empirically test our long-term causal effect estimator, demonstrating accurate estimation of long-term effects on synthetic data, as well as real data from the International Stroke Trial.

Although long-term effect estimation is our primary focus, the estimator and methods described can be applied to any single-stage causal effect. In this context, they can be interpreted as a novel  strategy that combines Front-Door and Instrument Variables to estimate causal effects in the presence of unobserved confounders.

In summary, our main contributions are:
\begin{enumerate}
    \item An algorithm for estimating long-term causal effects unbiasedly from both short-term experiments and observational data with latent confounders in linear structural causal models. This approach allows for continuous treatment variables---hence can deal with treatment dosages.
    \item  An analytical study of the variance of this estimator.
    \item  An extension of our estimator from linear structural causal models to partially linear structural models and a proof that unbiasedness still holds under a weak assumption.
    \item An empirical demonstration of our long-term causal effect estimator on synthetic and real data.
\end{enumerate}

\noindent Relevant source code and documentation has been made freely available in our \href{https://github.com/vangoffrier/UnConfounding}{online repository}.

%% file: 02-relatedwork.tex
\paragraph{\textbf{Estimating long-term causal effects}} The estimation of long-term causal effects from short-term experiments and observational data was initiated by \cite{athey2019surrogate}. The authors of that work devised a method to estimate such quantities by making use of short-term mediators, or surrogates, of the treatment. Their estimation strategy was comprised of two parts: first, they use the experimental data to determine the impact of the treatment on the surrogates, and then combined this impact with a predictive causal model that used the observational data to predict the impact of a change in the surrogates on the long-term outcome. This allowed them to predict the impact of the treatment on the long-term outcome directly at the end of the short-term experiment. However, this strategy only works if one assumes there are no unobserved confounders in the observational data. Recent work by \cite{cheng2021long} has expanded this approach with tools from machine learning, by learning efficient representations of the surrogates---again requiring there to be no unobserved confounders. More recent work by \cite{imbens2022long} has explored estimating long-term causal effects when unobserved confounders are present. These authors utilised results from the proximal causal inference literature, see \cite{tchetgen2020introduction} for an overview of these results, in their estimation strategy. However, to make use of these results, the authors have to assume existence of \emph{three} sequential mediators between the treatment and long-term outcome, and that these satisfy completeness conditions that, informally, require any variation in the latent confounders is captured by variation in the mediators. Our results, on the other hand, provide long-term treatment effect estimators that are unbiased even in the presence of latent confounders that do not require such sequential mediators that are strong proxies for the latent confounders.

\paragraph{\textbf{Combining experimental and observational data}} Beyond using observational data and short-term experimental data to estimate long-term causal effects, previous work has explored other advantages of combining observational and experimental data. Indeed, \cite{bareinboim2016causal} have investigated non-parametric identifiability of causal effects using both observational and experimental data, and how one can utilise such data regimes to transport causal effects learned in one data to another, in a paradigm they refer to as ``data fusion.'' Moreover, \cite{jeunen2022disentangling} has shown that one can learn to disentangle the effects of multiple, simultaneously-applied interventions by combining observational data with experimental data from joint interventions. Lastly, \cite{ilse2021efficient} demonstrated the most efficient way to combine observational and experimental data to learn certain causal effects. They showed they could significantly reduce the number of samples from the experimental data required to achieve a desired estimation accuracy.

\paragraph{\textbf{Linear structural causal models}} Many previous authors have worked in the linear structural causal model formalism. Indeed, \cite{shimizu2006linear} has shown that one can recover causal structure given just observational data if one assumes an underlying linear structural causal model with non-Gaussian noise. \cite{gupta2021estimating} has utilised this formalism to derive closed form expressions for the bias and variance of treatment effect estimators when both observed confounders and mediators are present. \cite{cinelli2019sensitivity} has derived closed-form expressions for the treatment effect bias when there are unobserved confounders in the dataset under investigation. Lastly, \cite{zhang2022causal} has explored what conditions lead to bias when estimating causal effects from non-IID data, and how can we remove such bias given certain assumptions. 

%% file: 03-new.tex
This section is structured as follows. We first define linear structural causal models with Gaussian noise, the class of models we will mainly be working with in this paper.
As a warm up to our main problem, we first explore long-term effect estimation when latent confounding influences the short-term treatment and long-term outcome, but does not influence the mediator. We note that this confounding may represent a single cause which persists through both short-term and long-term timescales.
The causal structure in this particular case corresponds to the front-door structure studied in \cite{pearl2009causality}. In this case, we derive---to our knowledge---a novel causal effect estimator for the front-door criterion, which may be of independent interest. This estimator is biased when latent confounding is present between the treatment and long-term outcome. However, the way the bias manifests is instructive, and suggests a way to adapt this estimation strategy to make it unbiased in this case. We prove that the estimator based on this strategy is indeed unbiased in the presence of latent confounding, and analytically study its variance. Finally, we extend this estimator from linear structural causal models to partial linear structural models, and prove that its bias is small in the presence of latent confounding if the treatment is strongly correlated with the latent confounder.

\subsection{Setting up the problem}

Motivated by the desire to unbiasedly combine short-term experimental data with long-term observational data, we define the following linear Gaussian structural causal model, which we will refer to as the linear confounded-mediator model (CMM):
\begin{equation}
    W_i = u_i^W, \hspace{15pt}
    X_i = dW_i + u_i^X, \hspace{15pt}
    M_i = cX_i + \epsilon W_i + u_i^M, \hspace{15pt}
    Y_i = aM_i + bW_i + u_i^Y,
    \label{struc_eqs}
\end{equation}

where index $i$ runs over samples. Here, $X$,$M$,$Y$,$W$ are respectively the treatment, short-term mediator, long-term outcome, and latent confounder. The causal structure for this model is depicted in Figure~\ref{fig:causalgraph_cm}. \footnote{In this work we assume the causal structure follows Figure~\ref{fig:causalgraph_cm}. To gain confidence in this assumption, one could employ causal discovery algorithms, see \cite{lee2017causal, dhir2020integrating, gilligan2022leveraging} for more information on these algorithms.} For the observed variables $X$,$M$,$Y$, the $u_i^N$ are independent Gaussian noise terms with zero mean: $u_i^N\sim \mathcal{N}(0,\sigma^2_{u^N})$ for node $N\in\{X,M,Y\}$. The term $u_i^W$ in the latent confounder structural equation is also an independent Gaussian noise term, but it has non-zero mean $\mu_{u^W}\neq 0$: $u_i^W\sim \mathcal{N}(\mu_{u^W},\sigma^2_{u^W})$.
\begin{figure}[t]
  \centering
  \input{fig/causalgraph_cm}
  \caption{%
Causal graph with mediator confounded by latent $W$.}
  \label{fig:causalgraph_cm}
\end{figure}
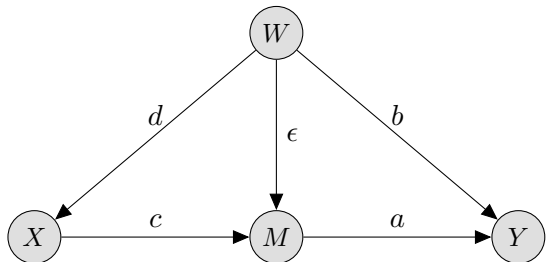

The framework having been defined, the typically desired treatment effect is $ac$.
But, as we assume that $c$ can be estimated unbiasedly from experimental data, our goal is to estimate $a$ given $c$ and an observational dataset of samples from $(X, M, Y)$.
That is, we ask to what extent it is possible to transfer knowledge of causation before a mediator to knowledge of causation after that mediator, in the presence of unobserved confounding on that mediator. For example, we could take $c$ to have been conclusively estimated via an A/B test, while $a$ is inaccessible to such experimentation due to its long timescale. This question also naturally arises in the context of chains of $N_M$ mediator variables, where the statistician hopes to propagate knowledge of an early mediation stage `down the chain'. Although we focus on scalar-valued variables throughout, an extension of this methodology to vector-valued $W$ and $M$ would be straightforward, only requiring an expansion of the covariance-matrix formalism outlined in Appendix \ref{sec:algebra} and interpreting $\epsilon$ as matrix-valued.

\subsection{Warm-up: a mediator without confounding}

With $\epsilon = 0$ the CMM in Figure~\ref{fig:causalgraph_cm} is the standard mediator---or front-door---model, treated thoroughly in the linear setting by \cite{gupta2021estimating}. It is well-known that so long as mediator $M$ is not directly confounded, $a$ may be unbiasedly estimated by the front-door criterion estimator (FDC):
\begin{equation}
    \hat{a}_{FDC} = P(Y|\text{do}(M)) = \sum_X P(Y|M, X)P(X) = \frac{(X.X) (M.Y) - (X.M) (X.Y)}{(X.X) (M.M) - {(X.M)}^2},
\end{equation}
where we have used $A.B$ as a shorthand for sample-space inner product $\sum_i A_i \cdot B_i$. Note that no knowledge of $c$ is needed. Indeed $c$ can be unbiasedly estimated by regressing $M$ on $X$ here.

We now give an alternative derivation of the FDC in terms of instrumental variables, a review of which is given in \cite{pearl2009causality}. Essentially, an instrument for a causal arrow $a: M \rightarrow Y$ is a variable $I$ such that a nonzero arrow $f: I \rightarrow M$ exists, and $I$ is uncorrelated with any other causes of $Y$, such as $W$ or $u_Y$ in the CMM. 

Consider the ordinary least squares (OLS) regression of $M$ on $X$, which trivially produces an unbiased estimator $\hat{c} = \frac{M.X}{X.X}$. Naively rearranging the structural equation, the residual of this estimator appears to be noise $u_M$. Constructing the true residual, we see that this still holds once all covariances are accounted for,
\begin{equation}
    R_c \equiv M - \textrm{OLS}[M|X] X,
\end{equation}
following from independence of $u_M$ from $u_X$ and $u_W$, the terminal causes of $X$. For the same reason, this residual $R_c = u_M$ is a valid instrument for $a: M \rightarrow Y$, as seen by constructing the relevant instrumental estimator:
\begin{equation}
    \hat{a}_{R_c} = \frac{\textrm{OLS}[Y|R_c]}{\textrm{OLS}[M|R_c]}.
\end{equation}
%
The above expression may be phrased entirely in terms of observed variables by making the substitution $u_M \mapsto M - \frac{M.X}{X.X} X$. Simplifying, we arrive at $\hat{a}_{R_c} = \hat{a}_{\text{FDC}}$. Hence, our instrumental-inspired estimator is unbiased and equal to the previously known estimator that follows from the front-door criterion. We will refer to $\textrm{Res}[M|X]$ corresponding to $c: X \rightarrow M$ more generally as the $c$-residual $R_c$. To our knowledge, this construction of the FDC via an instrumental estimator has not appeared in the literature, and we will refer to it as the Instrumental FDC (IFDC).

\subsection{The Instrumental FDC for confounded mediators}

A causal arrow $\epsilon: W \rightarrow M$ violates the conditions for the FDC. Our reason for introducing the IFDC is that it facilitates a natural extension of the FDC to the confounded mediator model, and more generally to any model with pathway $X \rightarrow M \rightarrow Y$ as a subgraph. The IFDC estimator can be presumed biased since $W$ and $M$ are no longer d-separated after conditioning on $X$. Expressions for the IFDC biases on $a$ (and corresponding OLS bias on $c$) are derived in Appendix \ref{sec:proofs} and are given by:
\begin{equation}
    \text{Bias}[\hat{c}_{\text{OLS}}] = \frac{d \epsilon \sigma_{u_W}^2}{d^2 \sigma_{u_W}^2 + \sigma_{u_X}^2} \hspace{15pt}
    \text{Bias}[\hat{a}_{R_c}] = \frac{b \epsilon \sigma_{u_W}^2 \sigma_{u_X}^2}{ \epsilon^2 \sigma_{u_W}^2 \sigma_{u_X}^2 + \sigma_{u_M}^2 (\sigma_{u_X}^2 + d^2 \sigma_{u_W}^2)}
    \label{bias_ars}
\end{equation}
The bias on $c$ vanishes if $d \gg \epsilon$ or $\sigma_{u_X}^2 \gg \sigma_{u_W}^2$, while the bias on $a$ vanishes if $\epsilon \gg b$, $\sigma_{u_M}^2 \gg \frac{b \epsilon}{d^2} \sigma_{u_X}^2$, or $\sigma_{u_M}^2 \gg b \epsilon \sigma_{u_W}^2$. From the structural equations, we might naively expect residual $\textrm{Res}[M|X] = u_M - \frac{\epsilon}{d} u_X$, and therefore explain the bias on $\hat{a}_{R_c}$ by the lack of independence between $u_X$ and $X$. However, computing the correlations of the residual with $X$ and $W$ in full reveals a surprise:
\begin{equation}
    \EE[\textrm{Cov}(R_c,X)] = \EE[M.X - \frac{M.X}{X.X} X.X] = 0
\end{equation}
\begin{equation}
    \EE[\textrm{Cov}(R_c,W)] = \frac{\epsilon}{d^2(N-1)} \left( \sigma_{u_X}^2 + \EE\left[ \frac{(X.u_X)^2}{X.X} \right] \right) > 0
\end{equation}
%
This is a lesson in not relying too heavily on the intuition of structural equations for confounding variables: the bias on $\hat{a}_{R_c}$ in fact arises entirely from correlation between $R_c$ and $W$. In the following we will see that the residual instrument can be modified to retain unbiasedness if $c$ is known.

\subsection{The \texorpdfstring{$\epsilon/d$}~-improved IFDC}

We propose that the most direct route to propagate improved knowledge of $c$ forward, in order to improve the IFDC estimator for $a$, is via intermediate knowledge of the quantity $\frac{\epsilon}{d}$. Ratios are desirable targets for estimation because they are insensitive to correlated biases on their numerator and denominator, and this particular ratio naively manifests in $R_c$ as controlling the size of the biasing $u_X$ term. We have identified several strategies for constructing estimators for $\frac{\epsilon}{d}$, with a ratio estimator based on $X = d W + u_X$ and the residual $M - c X \sim \epsilon W + u_M$ proving the most successful:
\begin{equation}
    \widehat{\left( \frac{\epsilon}{d} \right)} = \frac{\overline{M - c X}}{\bar{X}}
\end{equation}
where $\bar{A}$ denotes the sample mean $\sum_i(A_i) / N$. This estimator is unbiased in the limit of large samples, as $\mu_{u^W} \neq 0$ and $\mu_{u^X} = \mu_{u^M} = 0$. It is possible that superior estimators exist, but we find the ratio estimator to be adequate for our purposes.

The ``$\frac{\epsilon}{d}$-improved'' residual is then defined as the portion of $M$ which is leftover after removing all causal contributions from $X$, both via direct path $c$ and backdoor path $\epsilon/d$:
\begin{equation}
    R_R = R_c - \widehat{\left( \frac{\epsilon}{d} \right)} X = M - \left( c + \widehat{\left( \frac{\epsilon}{d} \right)} \right) X.
\end{equation} 
This construction leaves a door open to joint estimation of $c$ and $\frac{\epsilon}{d}$ from the prior stage in the model, in the sense that only the sum is needed and biases of opposite sign could destructively interfere, but we do not explore this further. The resultant instrumental estimator for $a$ takes the form:
\begin{equation}
 \hat{a}_{R_R} = \frac{R_R.Y}{R_R.M} = \frac{M.Y - \left( c + \widehat{\left( \frac{\epsilon}{d} \right)} \right) X.Y}{M.M - \left( c + \widehat{\left( \frac{\epsilon}{d} \right)} \right) X.M}.
\end{equation}
For convenience in application by the reader, we express our estimation strategy in algorithmic form:

\begin{algorithm}
\caption{$\frac{\epsilon}{d}$-improved Instrumental FDC Estimator}\label{alg:edIFDC}
\textbf{Input:} Short-term experimental dataset $\mathcal{E}=\{X, M\}$, observational dataset $\mathcal{O}=\{X, M, Y\}$ \\
\textbf{Output:} Estimator for causal effect of $M$ on $Y$.
\begin{algorithmic}[1]
\State From $\mathcal{E}$, estimate causal effect of $X$ on $M$: $c$.
\State Using samples from $\mathcal{O}$, regress $M$ on $X$ and compute residual: $R_c$
\State Using samples from $\mathcal{O}$, compute sample mean of $M-cX$ and $X$ and take their ratio: $\epsilon / d.$
\State Compute $R_c - \frac{\epsilon}{d}X$ and denote it by $R_R$.
\State Use $R_R$ in instrumental variable regression to estimate the causal effect of $M$ on $Y$. 
\end{algorithmic}
\end{algorithm}
In the next section we show this strategy unbiasedly estimates the causal effect of $M$ on $Y$.
\subsection{Unbiasedness and variance for the \texorpdfstring{$\epsilon/d$}~-improved IFDC}

Although we will argue via approximations and simulations that $\hat{a}_{R_R} = {R_R.Y}/{R_R.M}$ is unbiased (except at its pole), it is more straightforward to show that the ratio of estimators ${\EE(R_R.Y)}/{\EE(R_R.M)}$ is unbiased. This uncorrelated-ratio approximation is justified by the fact that it holds exactly for the IFDC, even in the presence of latent confounding, and is further discussed in Appendix \ref{sec:proofs}.

Evaluating algebraically by the methods outlined in Appendix \ref{sec:algebra} one obtains:
\begin{align}
    \mathbb{E} \left[ M.Y - \left( c + \frac{\epsilon}{d} \right) X.Y \right] &= a \left( \sigma_{u_M}^2 - \frac{c \epsilon \sigma_{u_X}^2}{d} \right), \label{bias_improved_a} \\
    \mathbb{E} \left[ M.M - \left( c + \frac{\epsilon}{d} \right) X.M \right] &= \sigma_{u_M}^2 - \frac{c \epsilon \sigma_{u_X}^2}{d}, \label{bias_improved}
\end{align}
and so we can observe that $\hat{a}_{R_R}$ is unbiased to the extent that the uncorrelated-ratio approximation holds. There is one exception: a unique value of $\frac{\epsilon}{d} = \frac{1}{c}$ exists (assuming homoscedasticity of the noise terms for simplicity) for which the numerator and denominator simultaneously approach $0$, and at which the bias is therefore unbounded. For finite sample sizes, one expects that this pole will be centered in a region of finite width where the estimator performs poorly, but that this region will contract to a delta function as $N \rightarrow \infty$. In summary, we have the following:
\begin{theorem}\label{theorem: linear}
In linear CMMs, the causal effect $a: M \rightarrow Y$ can be unbiasedly estimated by computing the following ratio of expectations:
$$
\frac{\mathbb{E} \left[ R_c.Y - \left(\frac{\epsilon}{d} \right) X.Y \right]}{\mathbb{E} \left[ R_c.M - \left( \frac{\epsilon}{d} \right) X.M \right]} = \frac{\mathbb{E} \left[ M.Y - \left( c + \frac{\epsilon}{d} \right) X.Y \right]}{\mathbb{E} \left[ M.M - \left( c + \frac{\epsilon}{d} \right) X.M \right]} = a.
$$
\end{theorem}
\begin{proof}
The result follows from application of \eqref{bias_improved_a} and \eqref{bias_improved}. Further details appear in Appendix \ref{sec:proofs}
\end{proof}
Although the presence of this isolated pole in the bias is not an overwhelming obstacle, it is practically inconvenient if samples are limited and one's system happens to fall in the wrong region of parameter space. Fortunately, there is one more tool at hand. In the case of a longer chain of mediators, more precisely if there exists a prior instrument on arrow $a: X \rightarrow M$ (which we will denote $g: V \rightarrow X$), it is no longer necessary for $c$ to be provided by an existing experiment. Instead, it may be estimated instrumentally by $\hat{c} = \frac{M.V}{X.V}$, while the $\frac{\epsilon}{d}$-improved IFDC can be built from an adjusted prior-instrument residual:
\begin{equation}
    R_V = Res(M|X) - \widehat{\left( \frac{\epsilon}{d} \right)} Res(X|V).
\end{equation}
The instrumental estimator $\hat{a}_{R_R}$ remains unbiased other than at a pole; but this pole is located at $\frac{\epsilon}{d} = \frac{1}{c \left( g^2 + 1 \right)}$, again assuming homoscedasticity of the noise terms. The practical consequence is that, if the practitioner has access to both a prior instrument and experimental data (or a low-variance estimation of $c$ from a previous link in the chain), they may choose whichever form of the IFDC is more suited to their value of $\epsilon/d$, which will be known. Given sufficiently strong prior causation $g$, the two poles are well-separated. However, even if only one of these is available, with sufficient samples the bias even arbitrarily near to a pole will approach 0.

Making use of the known variance properties of instrumental estimators, we construct an approximate expression for the asymptotic variance of $\hat{a}_{R_R}$ (details in Appendix \ref{sec:proofs}),
\begin{align}
    V_\infty (\hat{a}_{R_R}) &= \frac{b^2 \noise{W} \noise{X} + \noise{Y}(d^2 \noise{W} + \noise{X})} {(d^2 \noise{w} + \noise{X})} \cdot \frac{\noise{M} + \frac{\epsilon^2}{d^2} \noise{X}}{(\noise{M} - \frac{c \epsilon}{d} \noise{X})^2} \nonumber \\
    &= V_\infty (\hat{a}_{FDC}) \cdot \frac{ 1 + \frac{\epsilon^2}{d^2} \frac{\noise{X}}{\noise{M}} }{ \left( 1 - \frac{c \epsilon}{d} \frac{\noise{X}}{\noise{M}} \right) ^2}
\end{align}
which demonstrates that in general the improved estimator variance need not dramatically exceed that for the typical FDC, except near the bias pole $\frac{\epsilon}{d} = \frac{1}{c}$. Similarly, as treatment noise $\noise{X} \rightarrow 0$, $V_\infty (\hat{a}_{R_R}) \rightarrow V_\infty (\hat{a}_{FDC})$; this is equivalent to the situation where $d \gg \epsilon$ such that the treatment $X$ is very strongly coupled to the confounder $W$. As mediator noise $\noise{M} \rightarrow 0$, the variance vanishes, for the intuitive reason that weighted confounder $\epsilon W$ is then exactly known on a per-sample basis.

\subsection{Performance of improved estimators in a partial linear CMM}

We now assess to what extent the developed estimators remain unbiased when the causal effects $d: W \rightarrow X$ and $\epsilon: W \rightarrow M$ are permitted to be nonlinear. That is, we consider update to the confounded mediator model: $X = d(W) +u^X, M= c.X + \epsilon(W) + u^M.$ This is an example of a partial linear causal model, which we term the partial linear CMM. We will take functions $d(W)$ and $\epsilon(W)$ to be polynomial-valued, requiring further that $d(W)$ is invertible such that backdoor path $\epsilon \circ d^{-1}: X \rightarrow M$ is well-defined. Let us write:
\begin{equation}
    d(W) = \sum_{k=1}^\infty d_k \frac{W^k}{k!}, \hspace{15pt} \epsilon(W) = \sum_{k=1}^\infty \epsilon_k \frac{W^k}{k!}.
    \label{nonlinear_de}
\end{equation}
It is possible to define algebraic conditions on coefficients $d_k$ in the form of inequalities between the eigenvalues of the Hermite matrix of $d'(W)$, such that $d'(W) > 0 \hspace{2pt} \forall W$ (permitting $d'(W)=0$ at isolated points) such that $d(W)$ is invertible if and only if the algebraic conditions are satisfied.

It is well-known \citep{abramowitz1988} that the power series of an inverse function up to order $n$ may be computed iteratively from the coefficients of the original power series up to order $n$. We note however that each finite order in the original function induces nonzero terms to infinite polynomial order in the inverse function, which could be included say to order $m$ to improve precision. Taking $m = n$, we quote the series expansion for $\epsilon \circ d^{-1}$,
\begin{equation}
    \epsilon \circ d^{-1}(X) = \frac{\epsilon_1}{d_1} X + \frac{d_1 \epsilon_2 - d_2 \epsilon_1}{d_1^3} X^2 + \frac{d_1^2 \epsilon_3 + 2 d_2^2 \epsilon_1 - d_1 d_3 \epsilon_1 - 2 d_1 d_2 \epsilon_2}{d_1^5} X^3 + O(X^4),
    \label{eq:edivdpower}
\end{equation}
which also enjoys the key property that only coefficients order-by-order in $d$ and $\epsilon$ are needed.

We now investigate if the instrumental estimator $\hat{a}_{R_R}$, introduced in the linear case in the previous section and defined in the nonlinear case below, is biased:
\begin{equation}
    R_R = M - cX - \epsilon \circ d^{-1}(X) .
\end{equation}
Again, the justification for this estimation approach is that $R_R$ should be uncorrelated with confounder $W$, and therefore a good instrument for $a: M \rightarrow Y$, so long as it is possible to produce unbiased or low-bias estimates of $c$ and of the coefficients of $\epsilon \circ d^{-1}$.

In the non-linear case, how does one compute $\epsilon \circ d^{-1}(X)$? The residual from regressing $M$ on $X$ is naively given by $R = u_M + \epsilon \circ d^{-1}(X-u_X)$ by means of the backdoor path through $W$. Expanding the series representation from \eqref{eq:edivdpower}, we see that samplewise $R \rightarrow u_M + \epsilon \circ d^{-1}(X)$ as $u_X \rightarrow 0$, which corresponds to $\sigma_{u_x}^2 \rightarrow 0$. That is, to the case where $X$ is strongly correlated with $W$. Therefore, in this case, by polynomial regression of $R$ on $X$, it is theoretically possible to extract all coefficients of $\epsilon \circ d^{-1}$ to a desired order. We note that this method is much less sample-efficient than the ratio-based estimator for $\epsilon/d$ which we identified in the linear case. Now, in the case where $\sigma_{u_x}^2 \rightarrow 0$, where $X$ is strongly correlated with $W$, can we prove our estimation approach is unbiased? 

With $R_R$ well-defined, taking advantage of the structural equations, the bias on the instrumental estimator $\hat{a}_{R_R}$ may then be computed as:
\begin{align}
    \text{Bias}\left[ \hat{a}_{R_R} \right] &= \EE \left[ \frac{\left( u_M + \epsilon \circ d^{-1}(X - u_X) - \epsilon \circ d^{-1}(X) \right).Y}{\left( u_M + \epsilon \circ d^{-1}(X - u_X) - \epsilon \circ d^{-1}(X) \right).M} - a \right] \nonumber \\
    &= a\left(\frac{\noise{M} + \noise{X}P_1(\noise{X},\noise{W})}{\noise{M} + \noise{X}P_2(\noise{X},\noise{W})}\right) - a, \label{nonlinear_biasproof} 
\end{align}
where as in previous subsections, we have made use of the uncorrelated-ratio approximation to obtain an asymptotic bias estimate. $P_{1,2}$ are generic polynomials, and are computed algebraically by Isserlis' Theorem for higher-order moments. It is clear that this bias approaches $0$ as $X$ becomes increasingly correlated with $W$, yielding:
\begin{theorem}
In the partial linear CMM~\eqref{nonlinear_de}, $\text{Bias}\left[ \hat{a}_{R_R} \right] \to 0$ as $\noise{X} / \noise{M} \to 0$.
\end{theorem}
\begin{proof}
The result follows from \eqref{nonlinear_biasproof}.
\end{proof}
The assumption that $X$ is highly correlated with the latent confounder $W$ is not too strong. Indeed, the fact that $X$ and $W$ are causes of $M$ means that the confounding bias $W$ introduces between $M$ and $Y$ cannot naively be removed using back-door adjustment.

%% file: fig/causalgraph_cm.tex
\begin{tikzpicture}

  \node[obs] (x)  {$X$};
  \node[obs, right=2.5cm of x] (m)  {$M$};
  \node[obs, right=2.5cm of m] (y)  {$Y$};
  \node[obs, above=2.0cm of m] (w)  {$W$};

  \path (w) edge [->, >={triangle 45}] node[above] {$d$} (x) ;
  \path (w) edge [->, >={triangle 45}] node[above] {$b$} (y) ;
  \path (x) edge [->, >={triangle 45}] node[above] {$c$} (m) ;
  \path (m) edge [->, >={triangle 45}] node[above] {$a$} (y) ;
  \path (w) edge [->, >={triangle 45}] node[right] {$\epsilon$} (m) ;


\end{tikzpicture}

%% file: 04-experiments.tex
To empirically test our estimator in linear and partially-linear CMMs, we perform several experiments and measure prediction bias both as a function of confounding $\epsilon$ and of the noise variances $\sigma^2 \equiv \{ \noise{X}, \noise{M}, \noise{Y}, \noise{W}\}$. We first test on two synthetic datasets, one with linear data generation functions, and another with nonlinear data generation. To test in a more realistic setting, we create a semi-synthetic experiment using real data from the International Stroke Trial \cite{Carolei1997}. Initially, all couplings are assumed linear and are set to $1$ unless otherwise specified, and noises assumed zero-mean homoscedastic Gaussian, with the exception of $\mu_W = 1$. We will then relax the assumption of linearity on $d$ and $\epsilon$, and finally relax the assumption of Gaussianity on both $W$ and $X$ by generating semi-synthetic data from the International Stroke Trial dataset \cite{Carolei1997}. In all cases, we use the IFDC as baseline.

Relevant source code and documentation has been made freely available in our \href{https://github.com/vangoffrier/UnConfounding}{online repository}.

\subsection{Linear synthetic experiments}

First, we simulate the CMM and compare the performances of the IFDC and the $\epsilon/d$-improved IFDC in estimating $a$. A $30 \times 3$ grid over $\epsilon$ and $\sigma^2$ is specified, and at each point in parameter space, $10^6$ model samples are generated. A sample draw consists of first performing a random Gaussian draw from $\mathcal{N}(\mu,\sigma^2)$ for each noise component $u^N$, where $\mu = 1$ for $N=W$ and otherwise $\mu=0$, and second propagating this data through the structural equations \eqref{struc_eqs} with $a=b=c=d=1$. These samples are divided into $100$ runs, from which the mean and variance of $\hat{a}$ may be computed for each estimator. The results are shown in Figure~\ref{fig:scanenoise_linear}, with the IFDC shown in the left column and the $\epsilon/d$-improved IFDC in the right column. 
\begin{figure}[h!]
	\centering
	\includegraphics[scale=0.3]{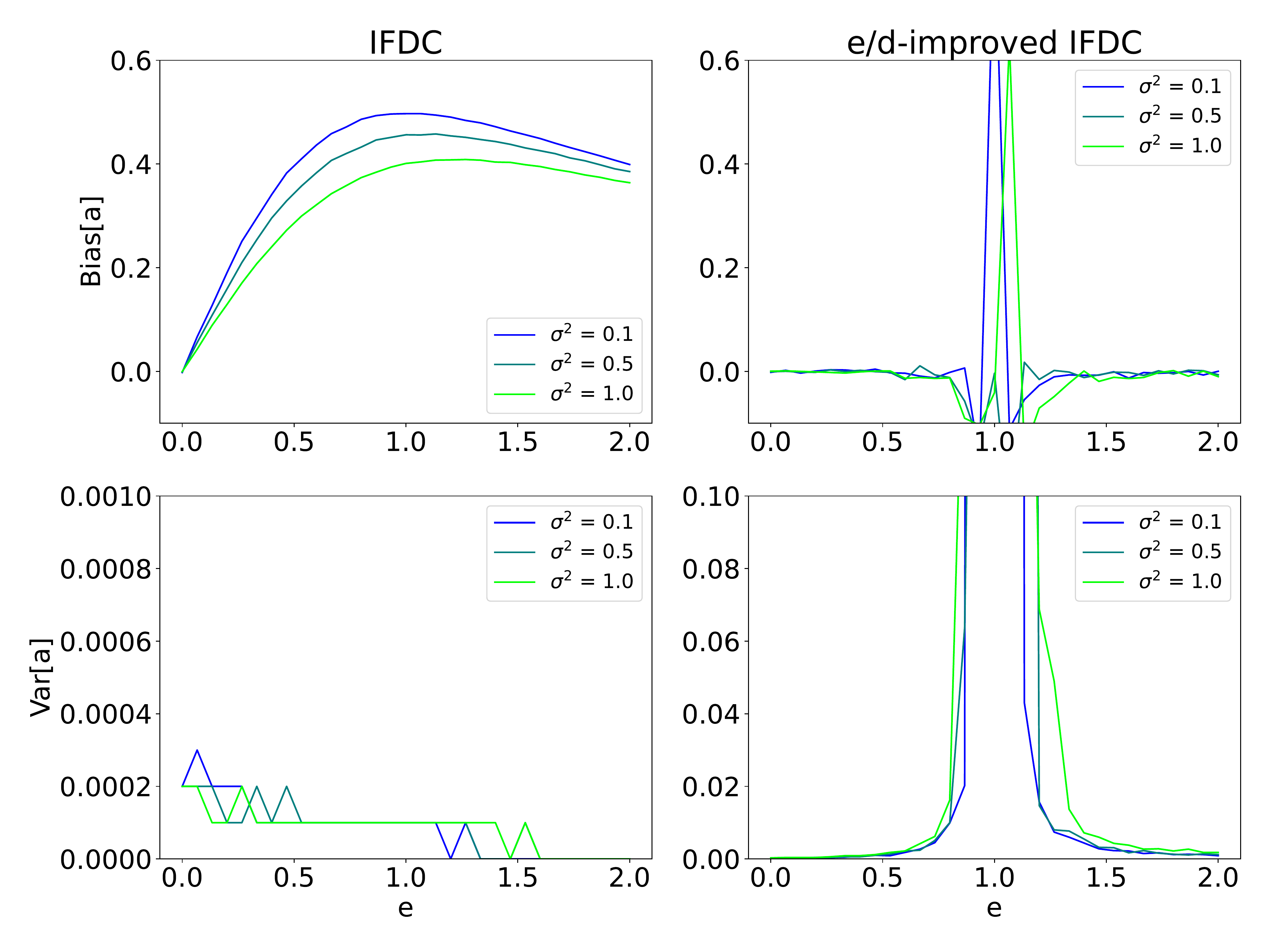}
	\caption{Experimental IFDC (left) and $\epsilon/d$-improved IFDC (right) biases (top) and variances (bottom) for a linear Gaussian model, plotted over $0 < \epsilon < 2.0$ and for $\sigma^2 \in \{ 0.1, 0.5, 1.0\}$. The plots show our estimator, the $\epsilon/d$-improved IFDC, is unbiased away from the pole at $\epsilon/d=1/c$, but the IFDC has high bias.}
	\label{fig:scanenoise_linear}
\end{figure}
The bias and variances properties for both estimators conform to our theoretical expectations. The nonzero bias from \eqref{bias_ars} is seen in the top left, with bias as $\epsilon$ for small $\epsilon$ and as $1/\epsilon$ for large $\epsilon$, while vanishingly small variance at this sample quantity is seen in the bottom left. For the improved estimator, the top right plot confirms unbiasedness throughout the $\epsilon$-domain except at pole value $\epsilon = 1$, as predicted by \eqref{bias_improved}, and reflected in the diverging variance precisely at this value on the bottom right. As mentioned in Section \ref{sec:theory}, the width of this bias pole can be improved with further samples, or alternatively can be translated by the introduction of a prior instrumental variable to $a: X \rightarrow M$

\subsection{Nonlinear synthetic experiments}

We now assess our perturbative approach to cubic-order nonlinearities in the coupling functions $d$ and $\epsilon$. A $6 \times 5$ grid over the quadratic and cubic polynomial coefficients is specified and at each point in parameter space, $10^5$ model samples are generated and divided into $100$ runs. We set $\sigma^2 = 0.3$ to ensure convergence, and $\epsilon = 2$ to avoid the bias pole for the improved estimator. The results are shown in Figure~\ref{fig:scand23_nonlinear} for cubic-polynomial $d$ and linear $\epsilon$, and in Figure~\ref{fig:scane23_nonlinear} for linear $d$ and cubic-polynomial $\epsilon$, with the IFDC shown in the left column and the $\epsilon/d$-improved IFDC in the right column.
\begin{figure}[h!]
	\centering
	\includegraphics[scale=0.2]{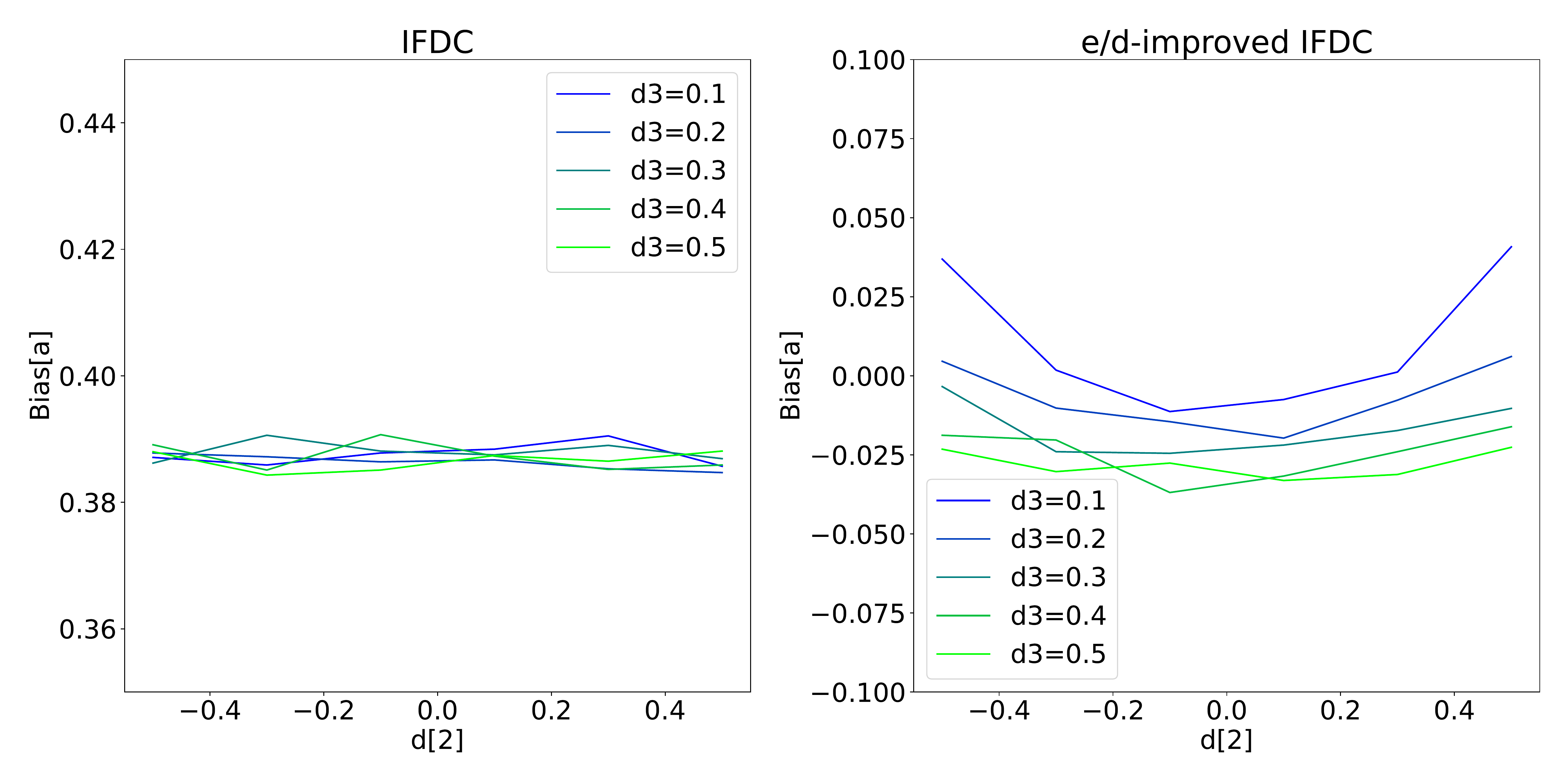}
	\caption{Experimental IFDC (left) and $\epsilon/d$-improved IFDC (right) biases for cubic-polynomial $d$ and linear $\epsilon$, plotted over $-0.5 < d_2 < 0.5$ and for $0 < d_3 < 0.5$. In the non-linear case, our estimator, the $\epsilon/d$-improved IFDC, has very low bias, but the IFDC has high bias (note that differing vertical scales have been employed to emphasize the trend).}
	\label{fig:scand23_nonlinear}
\end{figure}
\begin{figure}[h!]
	\centering
	\includegraphics[scale=0.2]{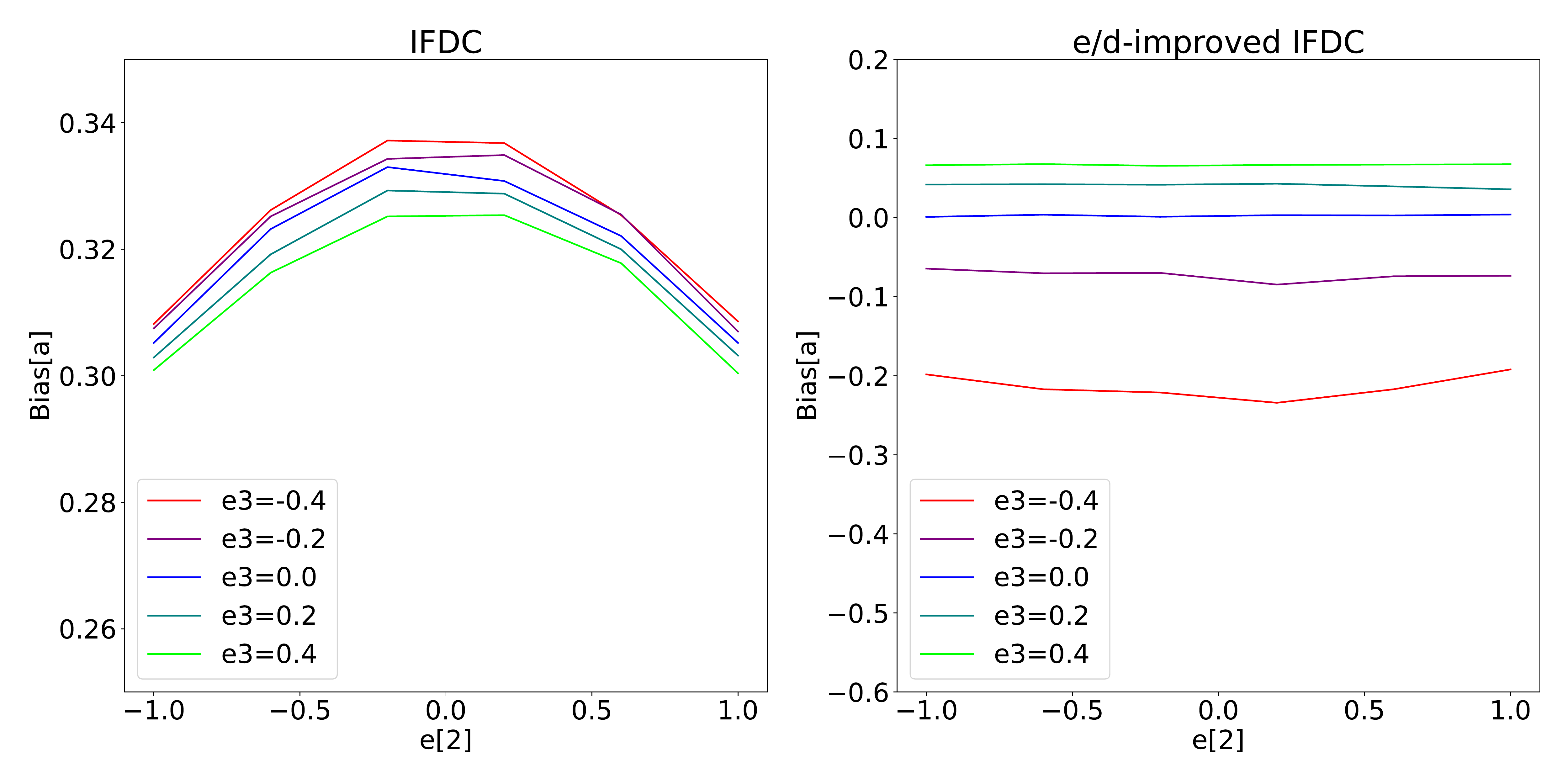}
	\caption{Experimental IFDC (left) and $\epsilon/d$-improved IFDC (right) biases for linear $d$ and cubic-polynomial $\epsilon$, plotted over $-1.0 < e_2 < 1.0$ and for $-0.4 < e_3 < 0.4$. Our estimator, the $\epsilon/d$-improved IFDC, has very low bias, but the IFDC has high bias (note that differing vertical scales have been employed to emphasize the trend).}
	\label{fig:scane23_nonlinear}
\end{figure}
For both nonlinear experiments, the $>0.35$ bias of the IFDC is drastically outperformed by the improved estimator with biases largely of magnitude $<0.05$. However, the IFDC enjoys significantly more stability against both quadratic and cubic nonlinearities, in fact appearing essentially insensitive to $d_2$ and $d_3$, compared with the improved estimator.

For the improved estimators, the dependence on acquired bias on the polynomial coefficients largely agrees with our theoretical analysis in Section \ref{sec:theory}. Comparing the right plot of Figure~\ref{fig:scand23_nonlinear} with Figure~\ref{fig:mathematica_d23}, we see confirmation both of the positive bias trend with $d_2$ and of the negative bias trend with $d_3$. There are, however, quantitative differences, where the perturbative approach overpredicts the bias by a factor of $5-10$, suggesting that a more evolved approach than Taylor expansion could be required to fully understand the consequences of nonlinearities in $d$.

Comparing the right plot of Figure~\ref{fig:scane23_nonlinear} with Figure~\ref{fig:mathematica_e23}, we again see confirmation both of the weak dependence of bias on $e_2$ and of the signed bias trend with $e_3$. Quantitatively, the match between theory and experiment is much stronger here, confirming the convergence of the $\epsilon$ polynomial expansion. For large, positive $e_3$, the numerical estimator begins to fail due to large variance, and more samples would be required to resolve this parameter region, but it is clear that beyond $e_3 \sim 0.4$ the improved estimator bias begins to surpass that of the original IFDC. In general we expect that higher-order nonlinearities would cause the estimator to fail more rapidly, although it is possible it might exceed expectations for specific nonlinear scenarios.

\subsection{International Stroke Trial semi-synthetic experiments}

To assess the performance of our estimators on more realistic data, we make use of the International Stroke Trial (IST) database \citep{Carolei1997}, a collection of stroke treatment and 14-day/6-month outcome data for $19,345$ individual patients. 

We take $W = AGE$ and $X = RSBP$, the systolic blood pressure at randomisation, both normalized to lie in $[0,1]$. We specify linear causal effects for $c,a,b,$ and $\epsilon$ and construct $M$ and $Y$ by propagation through the structural equations \eqref{struc_eqs} for each IST sample, including Gaussian random noise with variance $\sigma^2$. However, $d$ is not specified as it is manifest in the data with strength and linearity unknown. 

For simulation, a $20 \times 3$ grid over $\epsilon$ and $\sigma^2$ is specified, and at each point in parameter space, $200$ runs are generated using the same full set of $19,345$ IST samples, but with independently-sampled noises $u_M$ and $u_Y$. The bias results are shown in Figure~\ref{fig:scanenoise_ISTdata}, with the IFDC plotted with dashed lines and the $\epsilon/d$-improved IFDC with solid. Our improved estimator attains a generic improvement over the original IFDC for all $\epsilon \in [0,3]$ and $\sigma^2 \in [0.1,1]$, ranging between $20-40\%$ decrease in bias. This application is only a proof of concept, and these positive results indicate that further improvement could likely be achieved by more fully taking account of the non-Gaussianity of $X$ and $W$ and the nonlinearity of $d: X \rightarrow W$.
\begin{figure}[h!]
	\centering
	\includegraphics[scale=0.2]{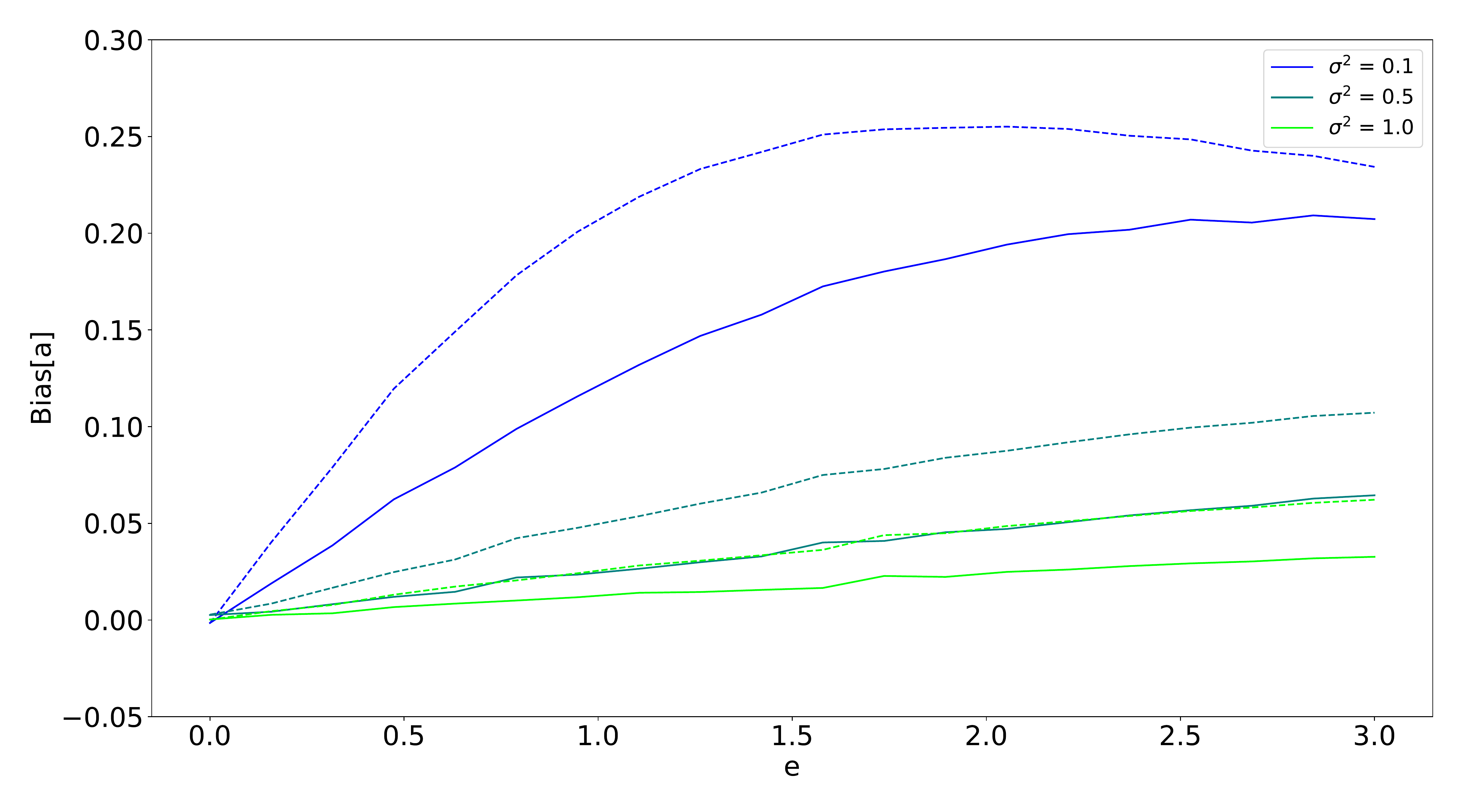}
	\caption{Experimental IFDC (dashed) and $\epsilon/d$-improved IFDC (solid) biases for synthetic IST data as described in the text, plotted over $0 < \epsilon < 3.0$ and with $\sigma^2 \in \{ 0.1, 0.5, 1.0\}$. In all cases, our estimator, the $\epsilon/d$-improved IFDC, has smaller bias than the baseline estimator.}
	\label{fig:scanenoise_ISTdata}
\end{figure}

%% file: 05-conclusion.tex
In this paper, we studied estimation of long-term treatment effects when both experimental and observational data were available. Specifically, we addressed the case where unmeasured confounders are present in the observational data. Our long-term causal effect estimator was obtained by combining regression residuals with short-term experimental outcomes in a specific manner to create an instrumental variable, which was then used to quantify the long-term causal effect through instrumental variable regression. We initially worked in the linear structural causal model framework,  proved this estimator is unbiased, and studied its variance. We then extended this estimator to partially linear structural models and proved unbiasedness still holds under a mild assumption. Finally, we empirically tested our long-term causal effect estimator on synthetic data, as well as real data from the International Stroke Trial---demonstrating accurate estimation. Although long-term effect estimation was our primary focus, the estimator and methods described could be applied to any single-stage causal effect with a nonzero-mean confounding variable; we therefore encourage that our results be interpreted within the much broader context of front-door and IV estimation methods.


%% file: A1-algebra.tex
In order to extend the derivations in \cite{gupta2021estimating} to cases with confounded mediators, multiple mediators, and pre-treatment covariates, it is necessary to introduce some new technology. Many key results including bias and variance for FDC-type estimators and covariance between estimators, all necessary to the estimation of the total causal effect, rely on essentially two steps. First, the desired expectation value is expanded using smoothing, also known as the law of total expectation or the tower rule:
\begin{equation}
    \EE[X] = \sum_x \sum_y x \cdot \mathbf{P}[X=x,Y=y] = \sum_y \left[ \sum_x x \cdot \mathbf{P}[X=x \mid Y=y] \right] \cdot \mathbf{P}[Y=y] = \EE[\EE[X \mid Y]],
\end{equation}
where $X$ and $Y$ are random variables (r.v.s) defined on the same probability space, and the expansion may be performed multiple times. In our application, $X$ is replaced by the desired expectation value, and a set of conditioners $\{Y\}$ are chosen so that the denominator (and as many numerator terms as possible) are fixed under $\{Y\}$. These fixed terms simplify by symmetry in some cases, and in more complex cases reduce to known distributions such as the Inverse-Wishart.

Second, the unfixed terms must be evaluated. Frequently these are of the form $\mathbf{E}[u,Y]$, where $u$ is some noise r.v. in the causal graph which is neither fixed by $Y$ nor independent from it. Linearity in a Gaussian-noise graphical model implies that any two node or noise r.v.s are bivariate normal, and indeed that any $N$ node or noise r.v.s are $N$-multivariate normal. This is hugely advantageous, because conditioning acts on a linear projection on a space of multivariate normal r.v.s.

For example, suppose $X$ and $Y$ have a bivariate normal distribution:
\begin{equation}
    (X,Y) \sim \mathcal{N}\left( \mathbf{\mu} = 
    \begin{pmatrix}
    \mu_X \\
    \mu_Y \\
    \end{pmatrix}, \mathbf{\Sigma} = 
    \begin{pmatrix}
        \sigma^2_X & \rho \sigma_X \sigma_Y\\
        \rho \sigma_X \sigma_Y & \sigma^2_Y\\
    \end{pmatrix} 
   \right),
\end{equation}
where $\rho$ is the correlation between $u$ and $Y$. Projection implies the following conditional expectations among $u$ and $Y$:
\begin{align}
    \EE[X \mid Y] &= \mu_X + \rho \frac{\sigma_X}{\sigma_Y}(Y-\mu_Y), \nonumber \\
    \EE[Y \mid X] &= \mu_Y + \rho \frac{\sigma_Y}{\sigma_X}(X-\mu_X), \nonumber \\
    \mathbb{V}[X \mid Y] &= \sigma^2_X (1-\rho^2), \nonumber \\
    \mathbb{V}[Y \mid X] &= \sigma^2_Y (1-\rho^2). \\
\end{align}
As a sanity check, we can see that both variances vanish if $\rho = 1$, and retain their independent values if $\rho = 0$. $\rho$ must be evaluated directly, which is straightforward in a linear Gaussian model; for instance, if $Y = \alpha X + U$ with $X \ind U$, $\text{cov}[X,Y] = \alpha \cdot \text{cov}[X,X] = \alpha \sigma^2_X$, implying $\rho = \frac{\text{cov}[X,Y]}{\sigma_X \sigma_Y} = \alpha \frac{\sigma_X}{\sigma_Y}$. This reproduces the well-known result that the conditional expectation of one of a set of summands on their sum is proportional to the ratio of their variances.

The above result is frequently sufficient, however it is too strict for our use case. We will need to be able to compute conditional moments of the form $\EE[ \prod_l u_i(l) \mid {Y_j}]$, where the product may include repeated or distinct noises, but the set $Y_j$ must be distinct (and sometimes may be reducible). To achieve this, we combine two tools: the general conditional projection for Gaussian families in terms of Schur complements, to easily handle a vector of conditioned r.v.s; and Isserlis' theorem for higher-order moments to handle arbitrarily complicated products of noises, so long as all r.v.s are zero-mean.

Following \cite{taboga2021}, the multivariate Gaussian conditional moments are: suppose vector-valued r.v. $X$ is $k$-multivariate normal with distribution
$X \sim \mathcal{N}(\mathbf{\mu}, \mathbf{\Sigma})$. Then for any partition $a+b = k$, where we define
\begin{equation}
X = \begin{pmatrix}
    X_a \\
    X_b \\
    \end{pmatrix}, 
\mathbf{\mu} = \begin{pmatrix}
    \mu_a \\
    \mu_b \\
    \end{pmatrix}, 
\mathbf{\Sigma} = \begin{pmatrix}
        \Sigma_a & \Sigma^T_{ab}\\
        \Sigma_{ab} & \Sigma_b\\
    \end{pmatrix},
\end{equation}
the vector-valued conditional mean is
\begin{equation}
\EE[X_a \mid X_b] = \mu_a + \Sigma^T_{ab} \Sigma^{-1}_b (X_b-\mu_b)
\end{equation}
and the matrix-valued conditional variance is
\begin{equation}
\mathbb{V}[X_a \mid X_b] = \Sigma_a - \Sigma^T_{ab} \Sigma^{-1}_b \Sigma_{ab}.
\end{equation}
Note that in the above conditional mean, only the bilinear survives if $\mu_a = \mu_b = 0$, as in our applications. Also, the term $\Sigma_a - \Sigma^T_{ab} \Sigma^{-1}_b \Sigma_{ab}$ is known as the Schur complement of block $\Sigma_b$ in $\mathbf{\Sigma}$. Without needing to rearrange the covariance matrix, $\mathbb{V}[X_b \mid X_a]$ can be found simply by taking the Schur complement of block $\Sigma_a$.

The complete partition above is excessive in most cases. If we only desire the expected mean for a single variable $X_i \in X_a$, for instance, the matrix equation becomes:
\begin{equation}
\EE[X_i \mid X_b] = [\mu_a]_i + [\Sigma^T_{ab}]_{ij} [\Sigma^{-1}_b]_{jk} [X_b-\mu_b]_k
\end{equation}
where $i,j,k$ are matrix indices and summations are assumed to be entire. What was a full $a \times a$ matrix multiplication is now a vector bilinear. Similarly, if we only desire a particular covariance $\text{cov}[X_i,X_j]$ for $X_i,X_j \in X_a$, the matrix equation becomes:
\begin{equation}
\text{cov}[X_i, X_j \mid X_b] = [\Sigma_a]_{ij} - [\Sigma^T_{ab}]_{im} [\Sigma^{-1}_b]_{mn} [\Sigma_{ab}]_{nj}
\end{equation}
where $i,j,m,n$ are matrix indices, and again we have arrived at a vector bilinear.

%% file: A2-proofs.tex
In this section we describe the construction of four residual-based instrumental estimators for $a$. Two of them will be shown to be unbiased except for some zero-measure choices of structural parameters, which will be characterised.

\subsection{Estimators and their Biases}

First let us recall the structural equations for the CM model,
\begin{equation}
    W_i = u_i^W, \hspace{15pt}
    X_i = dW_i + u_i^X, \hspace{15pt}
    M_i = cX_i + \epsilon W_i + u_i^M, \hspace{15pt}
    Y_i = aM_i + bW_i + u_i^Y,
\end{equation}
where all variables except the confounder $W$ and noises $u$ are taken to be observable.
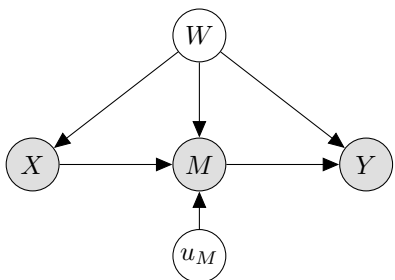
\begin{figure}[h!]
  \centering
  \input{fig/causalgraph_uminstcm}
  \caption{Causal graph with mediator confounded by latent $W$ and mediator noise term $u_M$.}
  \label{fig:causalgraph_uminstcm}
\end{figure}
The motivation for this approach is the observation that noise variable $u_M$, were it measurable, would be an acceptable instrument for $a$, as shown in Figure~\ref{fig:causalgraph_uminstcm}. The simplest approximation of $u_M$ is to regress $M$ on $X$ and take residual $R_c$ as an instrument. As discussed in the text, we could not expect $R_c$ to deliver an unbiased instrumental estimator for $a$, as the regression of $M$ on $X$ absorbs the backdoor path $X \leftarrow W \rightarrow M$. However it is instructive to compute the bias of the estimator by applying the law of total expectation conditioned on $X$ and $M$:
\begin{align}
    Bias[\hat{a}_{R_c}] &= \mathbb{E}[\mathbb{E}[\hat{a}_{R_c} - a|X,M]] \nonumber \\
    &= -\frac{b}{d} \mathbb{E}\left[ \mathbb{E}\left[ \left. \frac{M.(uX - X)  X.X - M.X X. (u_X - X)}{M.M X.X - (M.X)^2} \right| X, M \right] \right] \nonumber \\
    &= -\frac{b}{d} \mathbb{E}\left[  \frac{M.\mathbb{E}[u_X|X,M]  X.X - M.X X.\mathbb{E}[u_X|X,M]}{M.M X.X - (M.X)^2}  \right] \nonumber \\
    &= \frac{b \epsilon \sigma_{u_W}^2 \sigma_{u_X}^2}{ \epsilon^2 \sigma_{u_W}^2 \sigma_{u_X}^2 + \sigma_{u_M}^2 (\sigma_{u_X}^2 + d^2 \sigma_{u_W}^2)} \label{bias_ars_proof}
\end{align}
In the first line, the law of total expectation is applied, conditioned on $X$ and $M$ so as to isolate the numerator. In going from the first line to the second line, the independence of $u_Y$ from $X,M$ has been applied, and in the final expression the conditional expectation $\mathbb{E}[u_X|X,M]$ has been calculated as shown in Appendix \ref{sec:algebra}. Assuming homoscedasticity of the noise terms, this simplifies to
\begin{equation}
    Bias[\hat{a}_{R_c}] = \frac{b \epsilon}{1 + d^2 + \epsilon^2}.
\end{equation}
One notable property of this bias is that it vanishes in the limit of both small and large $\epsilon$, with global maximum bias of $\pm \frac{b}{\sqrt{4 + 2d^2}}$ at $\epsilon = \pm \frac{b}{2\sqrt{1 + d^2}}$ at $\epsilon = \pm \sqrt{1+d^2}$, as demonstrated in Figure~\ref{fig:scanenoise_linear}.
Not all estimators allow for this reduction strategy; in particular, the conditional expectations of the noise terms must combine just such that the denominator is cancelled and the expectation expression becomes that of a scalar. In such cases, we will proceed by estimating the numerator and denominator separately and treating the expectation of the ratio as well-approximated by the ratio of these expectations. For example, $\hat{a}_{R_c}$ has the following numerator and denominator expectations, derived by simple independence between noise terms and the fact that $\mathbb{E}[u_i.u_i] = \sigma_{u_i}^2$:
\begin{equation}
    \mathbb{E}[X.X M.Y - X.Y M.X] = \epsilon(b + a \epsilon) \sigma_{u_W}^2 \left(g^2 \sigma_{u_V}^2 + \sigma_{u_X}^2 \right) + a \sigma_{u_M}^2 \left(g^2 \sigma_{u_V}^2 + \sigma_{u_X}^2 + d^2 \sigma_{u_W}^2 \right);
\end{equation}
\begin{equation}
    \mathbb{E}[X.X M.M - (X.M)^2] = \epsilon^2 \sigma_{u_W}^2 \left( g^2 \sigma_{u_V}^2 + \sigma_{u_X}^2 \right) + a \sigma_{u_M}^2 \left(g^2 \sigma_{u_V}^2 + \sigma_{u_X}^2 + d^2 \sigma_{u_W}^2 \right).
\end{equation}
The ratio of these expectations, less $a$, delivers exactly the bias calculated in \eqref{bias_ars_proof}, which tells us that the numerator and denominator r.v.s are independent for this front-door estimator. In general this equivalence will fail due to correlations between the numerator and denominator, but we will assume the correlations to be weak as a useful first approximation.

We can now define and analyse the improved residuals which take advantage of prior-stage information about $c: X \rightarrow M$ and form the key results of this work. First, we take inspiration from the linear structural equations for the confounded mediator model, which suggest that the residual on $M$ after regression on $x$ should have the form $u_M - \frac{\epsilon}{d} u_X$. Taking the more general case of a prior instrument $g: V \rightarrow X$ in the CM model, we may arrive at this same linear structural quantity by the unique linear combination of residuals between $V$, $X$, and $M$ which removes unobserved data $W$, giving:
\begin{equation}
    R_V = Res[M,X] - \frac{\epsilon}{d} Res[X,V] \sim u_M - \frac{\epsilon}{d} u_X
\end{equation} 
where ratio $\frac{\epsilon}{d}$ is shown in Section \ref{sec:theory} to be unbiasedly estimable so long as confounder $W$ acquires some nonzero mean. Subsequently, we construct an instrumental estimator for a:
\begin{equation}
 \hat{a}_{R_V} = \frac{R_V.Y}{R_V.M} = \frac{V.V V.X (M.Y - \frac{\epsilon}{d} X.Y) + (V.X)^2 V.Y - \frac{\epsilon}{d}V.V V.M X.Y}{V.V V.X (M.M - \frac{\epsilon}{d} X.M) + (V.X)^2 V.M - \frac{\epsilon}{d}V.V V.M X.M}.
\end{equation}
Some simplification is achieved by applying the Law of Total Expectation conditioned over $V,X,M$, but the result is a nontrivial integral over these three heavily-correlated random vectors (in sample space):
\begin{align}
    Bias[\hat{a}_{R_V}] &= \mathbb{E} \Bigg[ \frac{b \sigma_{u_W}^2}{\sigma_{u_M}^2 \left( d^2 \sigma_{u_W}^2 + \sigma_{u_X}^2 \right) + \epsilon^2 \sigma_{u_W}^2 \sigma_{u_X}^2}  \nonumber \\
    &\left[ V.V V.X (M.M - \frac{\epsilon}{d} X.M) + (V.X)^2 V.M - \frac{\epsilon}{d}V.V V.M X.M \right]^{-1}  \nonumber \\
    & \bigg( d V.V M.X V.X \left( d \sigma_{u_M}^2 - c \epsilon \sigma_{u_X}^2 \right)  \nonumber \\
    &- \epsilon V.V V.X \left( X.X \left( d \sigma_{u_M}^2 - c \epsilon \sigma_{u_X}^2 \right) + \sigma_{u_X}^2 \left( d M.M - \epsilon X.M \right) \right)  \nonumber \\
    &- d V.V V.M \left( X.X \left( d \sigma_{u_M}^2 - c \epsilon \sigma_{u_X}^2 \right) + \epsilon X.M \sigma_{u_X}^2 \right)  \nonumber \\
    &+ \epsilon (V.X)^3 \left( d \sigma_{u_M}^2 - c \epsilon \sigma_{u_X}^2 \right) + \epsilon^2 V.M \sigma_{u_X}^2 (V.X)^2 \bigg) \Bigg] \label{bias_ared}
\end{align}
We do not yet know how to evaluate the above integral, except numerically. Instead, we can evaluate the expectations of the numerator and denominator:
\begin{equation}
    \mathbb{E} \left[ V.V V.X (M.Y - \frac{\epsilon}{d} X.Y) + (V.X)^2 V.Y - \frac{\epsilon}{d}V.V V.M X.Y \right] = a \left( \sigma_{u_M}^2 - \frac{c \epsilon (g^2 \sigma_{u_V}^2 + \sigma_{u_X}^2)}{d} \right);
\end{equation}
\begin{equation}
     \mathbb{E} \left[ V.V V.X (M.M - \frac{\epsilon}{d} X.M) + (V.X)^2 V.M - \frac{\epsilon}{d}V.V V.M X.M \right] = \sigma_{u_M}^2 - \frac{c \epsilon (g^2 \sigma_{u_V}^2 + \sigma_{u_X}^2)}{d}.
\end{equation}
In contrast to our results on $Bias[\hat{a}_{R_c}]$, the uncorrelated-ratio approximation suggests $Bias[\hat{a}_{R_V}] \simeq 0$. This only exactly holds if the integral in \eqref{bias_ared} evaluates to $0$, but is promising nonetheless. An intermediate possibility is that \eqref{bias_ared} approaches $0$ as $N_{samp} \rightarrow \infty$, but has a slow dependence on $N_{samp}$.

It is worth noting that $\hat{a}_{R_V}$ could have been constructed another way; naively from the structural equations, $Res[M,V] \sim \epsilon W + u_M$ just as $Res[M,X]$ does. We might even expect $Res[M,V]$ to experience less bias, since $V$ is not confounded by $W$. However, repeating the above analysis in the uncorrelated-ratio approximation gives a nonzero result,
\begin{equation}
    Bias[\hat{a}_{R_V}] \simeq \frac{bcd \sigma_{u_W}^2}{\sigma_{u_M}^2 + cd(cd+\epsilon) \sigma_{u_W}^2 + \frac{c(cd - \epsilon)}{d} \sigma_{u_X}^2},
\end{equation}
and so we have discarded this route.

It is straightforward to simplify estimator $\hat{a}_{R_V}$ and its corresponding residual to obtain the improved estimator $\hat{a}_{R_R}$ explored in-depth in the text. One simply sets $g=0$ to remove prior instrument $V$, and redefines the residual with $c$ presumed to be provided from an oracle:
\begin{equation}
    R_R = M - (c + \frac{\epsilon}{d}) X \sim u_M.
\end{equation} 
Importantly, this construction leaves the door open to joint estimation of $c$ and $\frac{\epsilon}{d}$ from the prior stage in the model, in the sense that only the sum is needed and biases of opposite sign could destructively interfere. The resultant instrumental estimator for $a$ is simple,
\begin{equation}
 \hat{a}_{R_R} = \frac{R_R.Y}{R_R.M} = \frac{M.Y - \left( c + \frac{\epsilon}{d} \right) X.Y}{M.M - \left( c + \frac{\epsilon}{d} \right) X.M}.
\end{equation}
Like $\hat{a}_{R_V}$, the full bias is not (yet) reducible beyond a high-dimensional integral,
\begin{align}
    Bias[\hat{a}_{R_R}] &= \mathbb{E} [ \frac{b \sigma_{u_W}^2}{\sigma_{u_M}^2 \left( d^2 \sigma_{u_W}^2 + \sigma_{u_X}^2 \right) + \epsilon^2 \sigma_{u_W}^2 \sigma_{u_X}^2}  \nonumber \\
    &\frac{\epsilon \sigma_{u_X}^2 M.M + (d \sigma_{u_M}^2 - \epsilon (2c + \frac{\epsilon}{d}) \sigma_{u_X}^2)X.M + (c + \frac{\epsilon}{d})(c \epsilon \sigma_{u_X}^2 - d \sigma_{u_M}^2)}{M.M - \left( c + \frac{\epsilon}{d} \right) X.M} ] \label{bias_arrem}
\end{align}
but, also like $\hat{a}_{R_V}$, this expectation appears unbiased in the uncorrelated-ratio approximation:
\begin{equation}
    \mathbb{E} \left[ M.Y - \left( c + \frac{\epsilon}{d} \right) X.Y \right] = a \left( \sigma_{u_M}^2 - \frac{c \epsilon \sigma_{u_X}^2}{d} \right);
\end{equation}
\begin{equation}
     \mathbb{E} \left[ M.M - \left( c + \frac{\epsilon}{d} \right) X.M \right] = \sigma_{u_M}^2 - \frac{c \epsilon \sigma_{u_X}^2}{d}.
\end{equation}
It is unsurprising that $R_R$ is no more biased than $R_V$, and we should expect that evaluation of the integrals in \eqref{bias_ared} and \eqref{bias_arrem} would show the same or better bias for $R_R$ even for finite sample size. In fact, numerical integration of \eqref{bias_arrem} indicates that any nonzero bias terms are proportional to $1/(N+k)$ for constants $k$, and therefore asymptotically vanish.

There is one crucial difference in the estimation performance of $\hat{a}_{R_R}$ vs. $\hat{a}_{R_V}$, a topological one arising from the presence of prior instrument $V$. As seen in the uncorrelated-ratio approximation, there are values of $\frac{\epsilon}{d}$ for which the numerator and denominator simultaneously approach $0$. Again assuming homoscedasticity of the noise terms for simplicity, this bias pole occurs at $\frac{\epsilon}{d} = \frac{1}{c}$ for $\hat{a}_{R_R}$ and at $\frac{\epsilon}{d} = \frac{1}{c \left( g^2 + 1 \right)}$ for $\hat{a}_{R_V}$. For finite sample sizes, one expects that each pole will be centered in a region of finite width where the estimator performs poorly, but that this bias will contract to a delta function as $N_{samp} \rightarrow \infty$. These poles are connected in the limit as $g \rightarrow 0$, although $\hat{a}_{R_V}$ is not defined at $g=0$. 

The practical consequence of the above analyses is that two instrumental estimators of $a$, constructed from the $\epsilon/d$-improved residual and from the remainder, are essentially unbiased. They each have a pole region of slowly-converging bias, however given sufficiently large $g$, these regions can be well-separated. In the presence of a prior instrument $g$, it is therefore possible to construct an unbiased estimator for $a$ throughout $(\epsilon,d)$ parameter space. It is for this reason that we illustrate both estimation strategies in full despite their obvious similarities.

\subsection{Variances}

We refer first to the variance computations in \cite{gupta2021estimating}, where finite-sample and asymptotic variances for $\hat{c}$ and $\hat{a}$ are calculated taking advantage of the asymptotic normality of OLS estimators, and the properties of inverse-Wishart-distributed matrices. For the front-door estimator, the asymptotic variances are quoted as follows:
\begin{equation}
    V_\infty (\hat{a}_{FDC}) = \frac{b^2 \noise{w} \noise{x} + \noise{y}(d^2 \noise{w} + \noise{x})} {(d^2 \noise{w} + \noise{x})\noise{m}},
\end{equation}
\begin{equation}
    V_\infty (\hat{c}) = \frac{\noise{m}} {d^2 \noise{w} + \noise{x}}.
\end{equation}
Via the Delta method, the asymptotic variance in estimating the total causal effect $ac$ is given by
\begin{equation}
    V_\infty (\hat{ac}) = c^2 V_\infty (\hat{a}) + a^2 V_\infty (\hat{c}),
\end{equation}
which holds as long as $Cov(\hat{a},\hat{c}) = 0$.

Following \cite{Corradi,Cameron}, the asymptotic variance for a scalar instrumental estimator $\hat{a}_{IV} = \frac{R.Y}{R.M}$ may be written
\begin{equation}
    V_\infty( \hat{a}_{R} ) = \frac{ \EE\left[ (R.R) \cdot \EE[\tilde{u}_Y.\tilde{u}_Y | R] \right] }{ Cov(R,M)^2 }
\end{equation}
where $\tilde{u}_Y$ denotes all additive contributions to $Y$ besides $a M$, and we have taken instrument $R$ to have zero mean. Following our claim that the instrumental estimator built from $R_c$ with no confounding on the mediator $(\epsilon = 0)$ is simply the FDC estimator, it is instructive to confirm that the IV asymptotic variance agrees with the FDC result from \cite{gupta2021estimating}.

For all causal structural models we consider, $\tilde{u}_Y = u_Y + b \cdot u_W$. In the $\epsilon = 0$ case, no confounding implies $R_c \ind \tilde{u}_Y$, so that $ \EE\left[ (R.R) \cdot \EE[\tilde{u}_Y.\tilde{u}_Y | R] \right] =  \EE\left[ R.R \right] \cdot \EE[\tilde{u}_Y.\tilde{u}_Y | R_c]$. Further computing $\EE[R_c.R_c] = \EE[R_c.M] = \noise{M}$, and evaluating $\EE[\tilde{u}_Y.\tilde{u}_Y | R_c]$ algebraically via the covariance matrix approach, we arrive at:
\begin{equation}
    V_\infty( \hat{a}_{R_c, \epsilon = 0} ) = \frac{ \noise{M} \cdot (\noise{Y} + b^2 \EE[u_W.u_W | R_c] ) }{ (\noise{M})^2 } = V_\infty (\hat{a}_{FDC}).
\end{equation}
When the mediator is permitted to experience some confounding $\epsilon$, we should expect some correlation between $R_c.R_c$ and $\tilde{u}_Y.\tilde{u}_Y$ via $u_W$. Separating this term from the product in the numerator, and observing that $\EE[R_c.R_c] = \EE[R_c.M] = \EE[\frac{M.M X.X - (M.X)^2}{X.X}]$, we find
\begin{equation}
    V_\infty( \hat{a}_{R_c} ) = \frac{ \noise{Y} + b^2 \EE[u_W.u_W | R_c] }{ \EE[\frac{M.M X.X - (M.X)^2}{X.X}] } + \frac{O(\sigma_{u_W}^4)}{\EE[\frac{M.M X.X - (M.X)^2}{X.X}]^2}
\end{equation}
where the quantity in the denominator has the distribution of the marginal from a Wishart-distributed matrix, as the quantity $\frac{1}{D}$ in \cite{gupta2021estimating}. It is possible to simplify this denominator expectation directly to only one non-trivial integral,
\begin{equation}
    \EE[R_c.R_c] = \noise{M} \cdot \frac{N}{N+2} + \epsilon^2 \noise{W} - \epsilon^2 \EE[\frac{\EE[(u_W.X)^2 | X]}{X.X}],
\end{equation}
where the final expectation value would reduce to $\noise{W} \cdot \frac{1}{N+2}$ were $u_W \ind X$, but numerical evaluation via cylindrical coordinates has confirmed that it approaches asymptotic $\noise{W}$ with strong correlation between $u_W$ and $X$. Thus $\EE[R_c.R_c]$ is bounded both above and below, with the overall $V_\infty( \hat{a}_{R_c} )$ slowly worsening as correlation between $u_W$ and $X$ becomes stronger.

If we assume that $c$ has been learned through previous experimentation, and that low-variance, unbiased estimation of $\frac{\epsilon}{d}$ has been attained, it is possible to obtain an exact variance result for the $\epsilon/d$-improved IFDC estimator. Since $R_c = u_M - \frac{\epsilon}{d} u_X$, $\EE[R_c.M] = \noise{M} - \frac{c \epsilon}{d} \noise{X}$ and $\EE[R_c.R_c] = \noise{M} + \frac{\epsilon^2}{d^2} \noise{X}$. Thus,
\begin{equation}
    V_\infty (\hat{a}_{R_R}) = \frac{b^2 \noise{W} \noise{X} + \noise{Y}(d^2 \noise{W} + \noise{X})} {(d^2 \noise{w} + \noise{X})} \cdot \frac{\noise{M} + \frac{\epsilon^2}{d^2} \noise{X}}{(\noise{M} - \frac{c \epsilon}{d} \noise{X})^2},
\end{equation}
which has the expected property that as $\epsilon \rightarrow 0$, $V_\infty (\hat{a}_{R_R}) \rightarrow V_\infty (\hat{a}_{FDC})$, but with asymptotic confounding $\epsilon \rightarrow \infty$, $V_\infty (\hat{a}_{R_R}) \rightarrow V_\infty (\hat{a}_{FDC}) \cdot \frac{\noise{M}}{c^2 \noise{X}}$. The variance expression only becomes unbounded at the pole $\frac{\epsilon}{d} = \frac{1}{c}$, just as expected from our computation of the bias.

%% file: fig/causalgraph_uminstcm.tex
\begin{tikzpicture}

  \node[obs] (x)  {$X$};
  \node[obs, right=1.5cm of x] (m)  {$M$};
  \node[obs, right=1.5cm of m] (y)  {$Y$};
  \node[latent, above=1.0cm of m] (w)  {$W$};
  \node[latent, below=0.5cm of m] (um)  {$u_M$};

  \edge {w}     {x} ;
  \edge {w}     {y} ;
  \edge {x}     {m} ;
  \edge {um}     {m} ;
  \edge {w}     {m} ;
  \edge {m}     {y} ;

\end{tikzpicture}

%% file: A3-nonlinear.tex
As two practical examples, we demonstrate the computed $\epsilon/d$-improved IFDC biases for cubic-polynomial $d$ and linear $\epsilon$, and for linear $d$ and cubic-polynomial $\epsilon$. Specifically,
\begin{equation}
    d(W) = d_1 W + d_2 W^2 + d_3 W^3,
\end{equation}
in which case the invertibility condition simplifies to $-\sqrt{3 d_1 d_3} \leq d_2 \leq \sqrt{3 d_1 d_3}$, which may only be fulfilled if $d_1$ and $d_3$ have the same (or 0) sign. Setting all variances and $b=c=d_1=\epsilon_1=1$ for simplicity, we find
\begin{align}
    Bias\left[ \hat{a}_{R_R,n_d = 3} \right] &= \frac{6(2d_2^2 - d_3)(1 + 3d_3)}{1 + 72d_2^4 - 30d_3 - 108d_3^2 - 180d_3^3 + 18 d_2^2 (3 + 10d_3 + 20d_3^2)}, \\
    Bias\left[ \hat{a}_{R_R,n_\epsilon = 3} \right] &= \frac{3\epsilon_3}{\epsilon_2^2 + 12\epsilon_3 + 9\epsilon_3^2}.
\end{align}
Varying the cubic coefficient and plotting curves over the quadratic coefficient, theoretical bias estimates for these two scenarios are presented in Figures \ref{fig:mathematica_d23} and \ref{fig:mathematica_e23}, respectively. We have set all noise variances to $\sigma^2 = 0.2$ for these computations, in order to more clearly show trends and to assure convergence. For Figure \ref{fig:mathematica_d23}, we have taken terms of $d^{-1}$ up to order $m=10$ to demonstrate that at this order in the expansion, the prediction still varies substantially; it is ``non-perturbative'', and so even to order 10 should only be taken as a qualitative estimate. The convergence up to order 3 in Figure \ref{fig:mathematica_e23}, however, is taken to be sufficiently precise.
\begin{figure}[h!]
	\centering
	\includegraphics[width=0.9\textwidth]{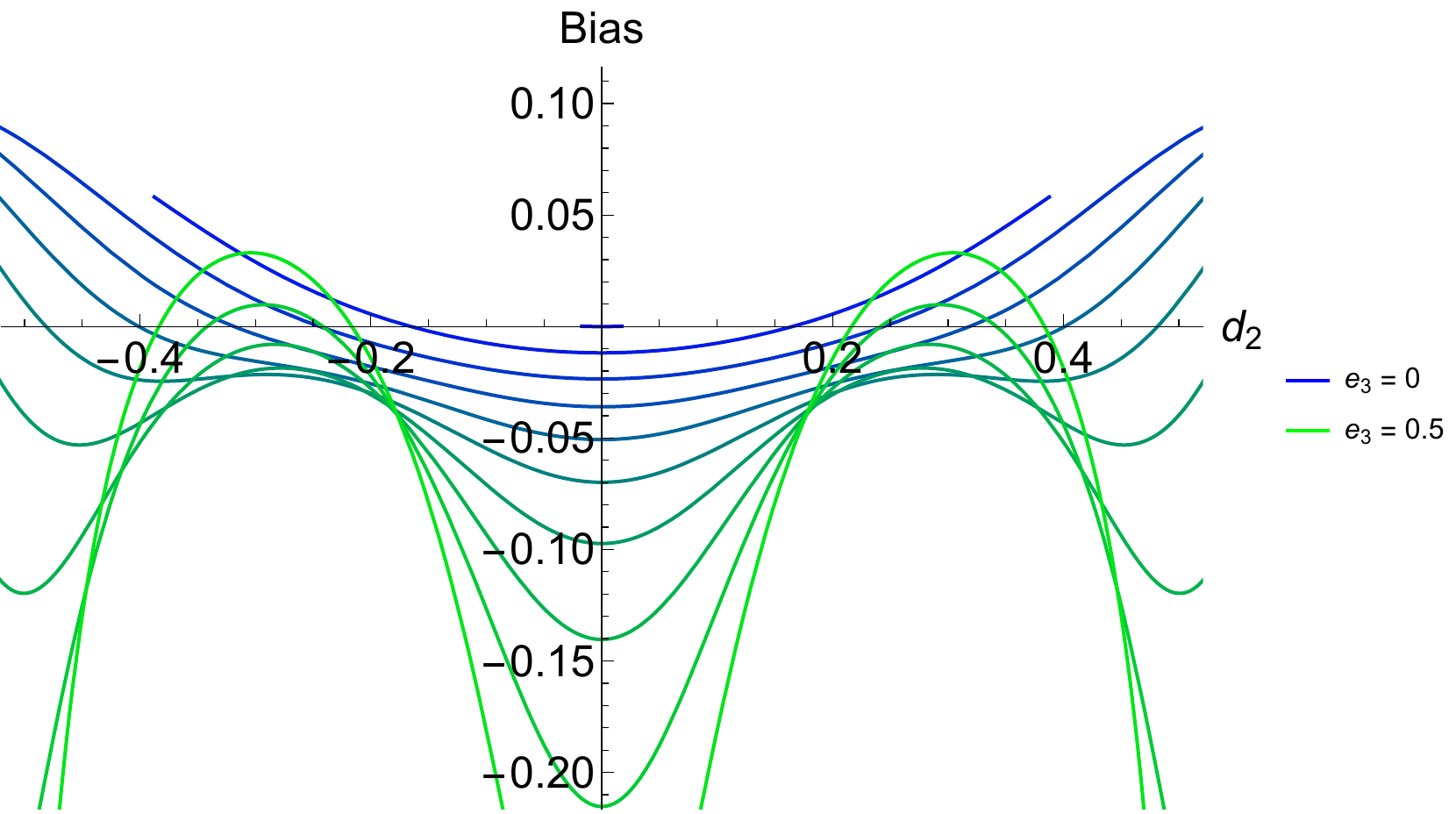}
	\caption{Theoretical $\epsilon/d$-improved IFDC biases for cubic-polynomial $d$ and linear $\epsilon$, plotted over $-0.5 < d_2 < 0.5$ and with curves ranging over $0 < d_3 < 0.5$.}
	\label{fig:mathematica_d23}
\end{figure}
\begin{figure}[h!]
	\centering
	\includegraphics[width=0.9\textwidth]{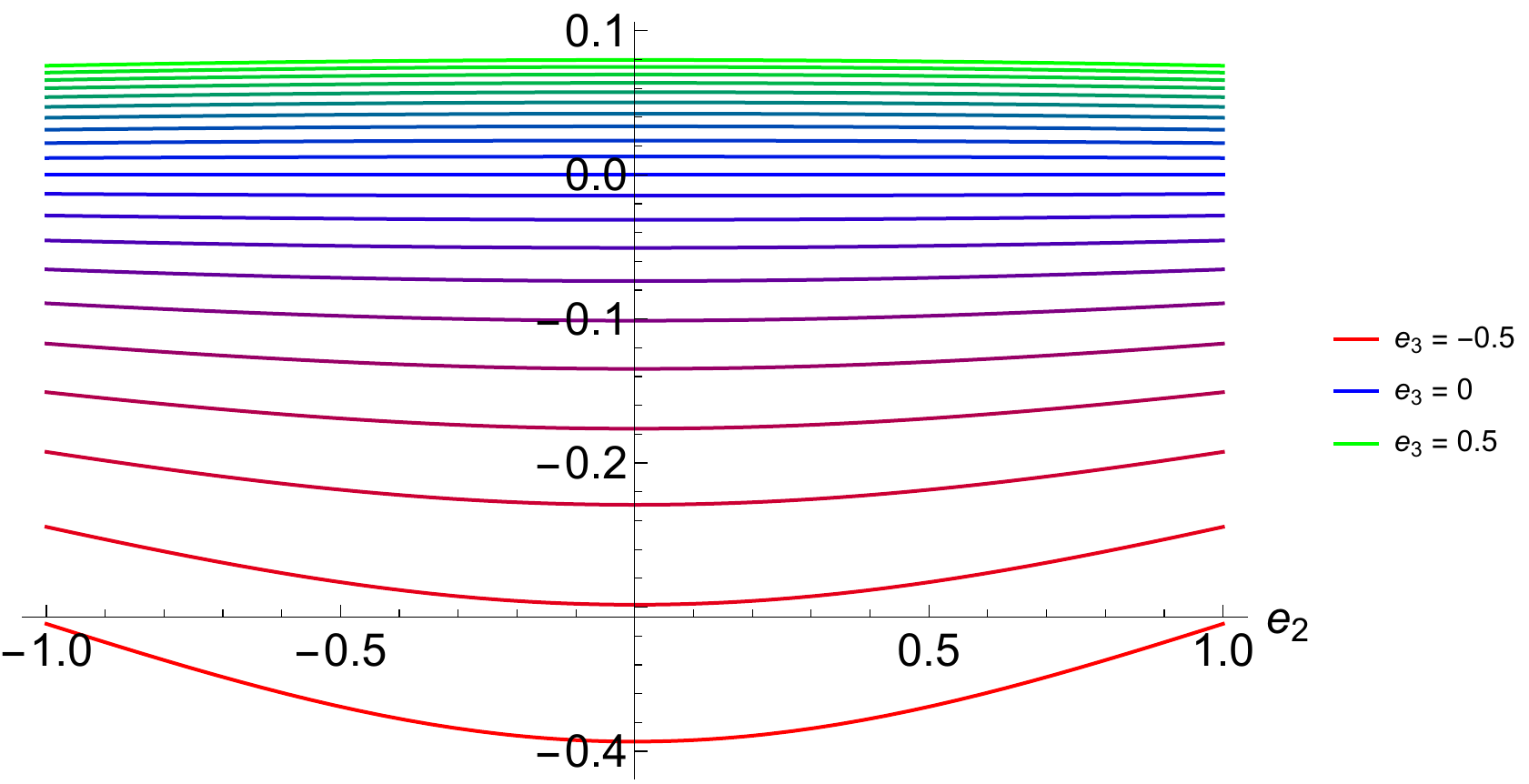}
	\caption{Theoretical $\epsilon/d$-improved IFDC biases for linear $d$ and cubic-polynomial $\epsilon$, plotted over $-1.0 < \epsilon_2 < 1.0$ and with curves ranging over $-0.5 < \epsilon_3 < 0.5$.}
	\label{fig:mathematica_e23}
\end{figure}
To summarise these results, up to non-perturbative effects we expect that nonzero $d_2$ pushes the bias in the positive direction, while nonzero $d_3$ (restricted to be positive by invertibility) pushes the bias in the negative direction. Coordinates in $(d_2,d_3)$ where unbiasedness is retained or nearly retained should therefore exist. In contrast, nonzero $\epsilon_2$ appears to have a much smaller impact on bias, in fact tending towards $0$, while nonzero $e_3$ leads to bias in the direction of $sign(\epsilon_3)$. It is noteworthy that for positive $\epsilon_3$, the bias is small and tentatively approaches an asymptote around $0.1$, while for negative $\epsilon_3$, bias grows rapidly and appears unbounded.